\documentclass[letterpaper, 10 pt, conference]{ieeeconf}  

\IEEEoverridecommandlockouts                              

\usepackage{graphics} 
\usepackage{epsfig} 
\usepackage{amsmath} 
\usepackage{amssymb}  
\DeclareMathOperator*{\argmin}{arg\,min}
\usepackage{algorithm}
\usepackage{algorithmic}

\usepackage{caption}
\usepackage{gensymb}
\usepackage{multirow}
\usepackage{verbatim}
\newtheorem{theorem}{Theorem}[section]

\usepackage[caption=false,subrefformat=parens,labelformat=parens]{subfig}
\usepackage{esint}

\title{
Efficient Probabilistic Collision Detection for Non-Convex Shapes
}

\author{Jae Sung Park, Chonhyon Park, Dinesh Manocha
\thanks{Jae Sung Park and Chonhyon Park and Dinesh Manocha are with the Department of Computer Science, University of North Carolina at Chapel Hill. E-mail: {\tt\small \{ jaesungp, chpark, dm\}@cs.unc.edu}.}
}

\begin{document}

\maketitle
\thispagestyle{empty}
\pagestyle{empty}

\begin{abstract}
We present new algorithms to perform fast probabilistic collision queries between convex as well as non-convex objects. Our approach is applicable to general shapes, where
one or more objects are represented using Gaussian  probability distributions. We present a fast new algorithm for a pair of convex objects, and extend the approach to non-convex models using hierarchical representations. We highlight the performance of our algorithms with various convex and non-convex shapes on complex synthetic benchmarks and trajectory planning benchmarks for a 7-DOF Fetch robot arm.  
\end{abstract}


\section{Introduction}

Collision detection is an important problem in many applications, including physics-based simulation and robotics.
In robot motion planning, collision detection is regarded as one of the major bottlenecks.
There is extensive work on faster collision checking for convex shapes, hierarchical algorithms, and methods for deformable models~\cite{gilbert1988fast,Lin03collisionand,teschner2005collision}.
These prior collision detection techniques assume an 
exact representation of the geometric objects. They perform exact interference tests between the primitives and return the overlapping features.

In many applications, including robotics, virtual environments, and dynamic simulation, exact representations of the primitives are not easily available. Rather, the object representations are described using probability distribution functions. This may occur because the environment data is captured using sensors and only partial observations are therefore available. Furthermore, the primitives captured or extracted using sensors tend to be noisy.  In this case, the goal is to compute the {\em collision probability} of two or more objects when one or more object representations (e.g. positions, orientations, etc.) are represented in terms of probability distributions~\cite{rusu2009real,bae2009closed,pan2011probabilistic}. 
In many robotics applications, such probabilistic collision queries are performed on the imperfect representations due to the uncertainties.
For example, the planning of robot motions in dynamic real-world environments has to compute safe robot trajectories that avoid collisions with moving obstacles.
In this case, the future obstacle positions are not known exactly and typically predicted using probability distributions. In other cases, the environment is typically represented using point clouds with some error distribution.

\begin{figure}[ht]
  \centering
  \subfloat[][]
  {
    \includegraphics[width=\linewidth]{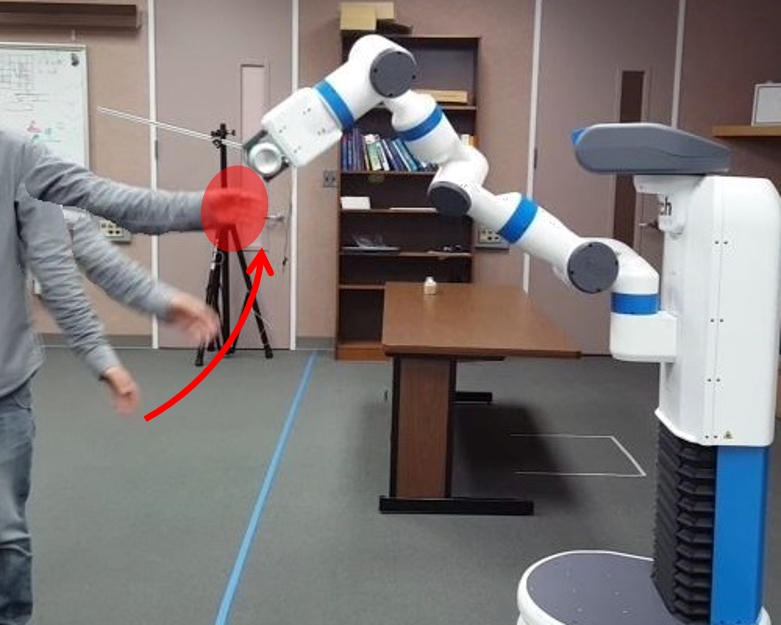}
  }
  \caption{{\bf Trajectory planning for robot avoiding human arm  using  our probabilistic collision algorithm:} The moving human arm is tracked using a point cloud sensor and the positional error is represented as a Gaussian probability distribution. The $95\%$ boundary of Gaussian distribution is represented by the red circle. We use our novel probabilistic algorithm to guarantee that the collision probability is less than $5\%$ at any state during the given trajectory. In order to handle such non-convex objects and compute tight bounds for probabilistic collision detection, we found that a hierarchy of OBBs (oriented bounding boxes)  provides the best results in terms of running time and probabilistic bounds.
 }
  \label{fig:robot_trajectory}
\end{figure}

Typically, these probability distributions are approximated
using Gaussian distributions and the collision computations are performed using  probability distribution functions (PDF). The resulting probabilistic collision
detection for objects with Gaussian distributions is defined
using a confidence level, and probability can be evaluated using numerical
integration or stochastic techniques. However, accurate techniques based on Monte Carlo integration are  too slow for realtime applications. Current algorithms for probabilistic computations are based on circular or spherical approximations ~\cite{du2011probabilistic,park2016fast,pan2013real}, but they tend to be rather conservative  (in terms of bounds) for general non-convex shapes.

\noindent {\bf Main Results:} 
In this paper, we present a fast probabilistic collision detection algorithm for general non-convex models. Our approach is applicable to any models represented in terms of inexact polygons or noisy point cloud data in which we are given two general shapes and the error is represented using Gaussian probability distributions. 
We compute hierarchical representations of non-convex models using simpler bounding volumes and present a novel hierarchical probabilistic collision detection algorithm. Moreover, we present fast and reliable algorithms  to perform probabilistic collision detection on simple shapes such as AABBs (axis-aligned bounding boxes), OBBs (oriented-bounding boxes), k-DOPs (k discretely oriented polytopes), and convex polytopes.
We have evaluated the performance of our probabilistic collision detection on many synthetic and real-world benchmarks (captured using the Kinect). Moreover, we have integrated them with trajectory planning, and highlight the performance for real-time motion prediction and planning for a 7-DOF Fetch robot. We also evaluate them on complex synthetic benchmarks  and observe that the hierarchies of OBBs or k-DPS provide the best balance between tight bounds and running times.

The rest of the paper is organized as follows.
We survey prior work on exact and probabilistic collision detection algorithms in Section~\ref{sec:related}.
Section~\ref{sec:overview} gives an overview of the problem of probabilistic collision detection, and we present our algorithm for convex and non-convex polygonal models in Section~\ref{sec:pcd}. We analyze these algorithms and highlight their performance on complex benchmarks in Section~\ref{sec:result}.
\section{Related Work}
\label{sec:related}

In this section, we give a brief overview of prior work on exact and probabilistic collision detection algorithms.

\subsection{Collision Detection using Bounding Volume Hierarchies}

There is extensive work on exact collision detection algorithms for geometric models represented using triangles or higher order primitives. These include efficient algorithms for convex polytopes~\cite{gilbert1988fast,lin1991fast} and general algorithms for non-convex shapes using bounding volume hierarchies~\cite{Lin03collisionand,ericson2004real}.
Many techniques have been proposed to improve the efficiency of collision detection using Bounding Volume Hierarchies (BVHs).
Some of the commonly-used hierarchies are based on bounding volumes such as spheres~\cite{hubbard1993interactive,he2014efficient} or axis-aligned bounding boxes (AABBs)~\cite{bergen1997efficient}.
Other algorithms that use tighter-fitting bounding volumes include oriented-bounding boxes (OBBs)~\cite{gottschalk1996obbtree}, discrete oriented polytopes (k-DOPs)~\cite{klosowski1998efficient}, or their hybrid combinations~\cite{sanna2004cdfast}.

There is a trade-off between different bounding volume algorithms.
The simple bounding volume algorithms have lower overhead in terms of overlap test, but can result in a high number of false positives and exact tests between the primitives. 
Tighter-fitting bounding volume algorithms may involve more complex overlap tests, but the tighter bounds reduce the number of false positives. The relative performance of different BVH-based algorithhms varies based on the complexity and shape of geometric models and their relative configurations.

\subsection{Probabilistic Collision Detection}

A common approach to checking for collisions between noisy geometric datasets is to perform exact collision checking using enlarged volumes that enclose the original object primitives~\cite{van2012lqg,Park:2012:ICAPS}. 
Many approaches handling point cloud sensor data convert the data  into a set of boxes and check for collisions between the boxes and a robot~\cite{rusu2009real}.
Other approaches generally enlarge the object bounding volumes to compute a new bounding volume for a given confidence level. This may correspond to a sphere~\cite{bry2011rapidly} or a ‘Sigma hull’~\cite{lee2013sigma}.
However, the computed volume overestimates the collision probability.

Other approaches have been proposed to perform probabilistic collision detection on point cloud data.
Bae et al.~\cite{bae2009closed} presented a closed-form expression for the positional uncertainty of point clouds.
Pan et al.~\cite{pan2011probabilistic} reformulated the probabilistic collision detection problem as a classification problem and computed per point collision probability. However, these approaches assume that the environment is mostly static.
For realtime collision detection, probabilistic collision detection is performed using broad phase data structures that handle large point clouds~\cite{pan2013real}.

Some approaches approximate the collision probability using numerical integrations or stochastic techniques~\cite{blackmore2006probabilistic,lambert2008fast}, which require a large number of sample evaluations to compute an accurate probability. 
Guibas et al.~\cite{guibas2010bounded} evaluate collision probability bounds using numerical integrations in multiple resolutions, which can avoid unnecessary numerical integrations in the high-resolution.
The collision probability between objects with Gaussian distributions can be approximated using the probability at a single configuration, which corresponds to the mean~\cite{du2011probabilistic} or the maximum~\cite{park2016fast} of the probability distribution function (PDF).

\section{Overview}
\label{sec:overview}

In this section, we define the symbols and notation used in the rest of the paper and give the problem statement of probabilistic collision detection.

\subsection{Symbols and Notation}

We use upper case symbols, like $A$ and $B$, to denote 3D input primitives. 
We use boldface letters, such as $\mathbf{p}$ or $\mathbf{x}$, to represent vectors.
The probability that an \textit{event} can occur is denoted as $P(\textit{event})$.
We represent Gaussian distribution as $\mathcal{N}(\mathbf{p}, \Sigma)$, where $\mathbf{p}$ is a mean vector and $\Sigma$ is a covariance matrix.
The probability density function of the Gaussian distribution $\mathcal{N}(\mathbf{p}, \mathbf{\Sigma})$ is represented as $p(\mathbf{x}, \mathbf{p}, \mathbf{\Sigma})$.


\subsection{Background}
Probabilistic collision detection algorithms are used to perform collision checking between two or more objects when the objects are not represented exactly and some of the input information such as positions or orientations of polygons or point clouds are given as probability distributions. 
The output could be a binary answer corresponding to whether or not those objects overlap. In another case, it could be a probability value that corresponds to the probability that an overlap occurs for the given input. This probability value is useful in many applications where a binary value is not enough to evaluate the given input.

In many applications, probability distributions are approximated using Gaussian distributions, which allow efficient computation using probability distribution functions (PDF). However, Gaussian distributions have non-zero probabilities in the entire problem space, which always result in non-zero collision probabilities. Therefore, the probabilistic collision detection for objects with Gaussian distributions is defined using a \emph{confidence level}, which is a threshold probability that determines the binary output of the collision detection (active collision or not) from the computed collision probability (e.g. 0.99).

In general, the exact collision probability is not computable even for a simple collision scenario between circles with a Gaussian distribution (see Sec.~\ref{subsec:cpa}). The probability can be approximated using a numerical integration or a stochastic technique such as the Monte Carlo method~\cite{blackmore2006probabilistic,lambert2008fast}, which requires a large number of sample evaluations to compute an accurate probability.

Therefore, efficient approximation methods have been proposed to compute the collision probability with tight bounds, without the evaluation of a large number of samples~\cite{du2011probabilistic,park2016fast}, for a limited type of bounding volumes such as circles or spheres. In this paper, we show that our novel algorithms based on OBBs or k-DOPs provide better solution

\subsection{Problem Statement}
We assume that the two input primitives, $A$ and $B$, are each represented as a polygon or triangle mesh. We also assume that one Gaussian distribution (PDF) is available for each object.
In this paper, we consider the cases when $A$ and $B$ are both convex polytopes, or general  non-convex shapes. If the input is given as 3D point cloud data with some error distribution, we pre-compute a mesh representation and the appropriate Gaussian distribution from the point cloud data.

A positional displacement vector ${\bf \epsilon}$ is applied as translational operator to the volume $A$, which is denoted as $A + \mathbf{\epsilon} = \lbrace \mathbf{a} + {\bf \epsilon} | \mathbf{a} \in A \rbrace$. We assume a Gaussian distribution assumption on $\epsilon$ with zero mean and covariance matrix $\mathbf{\Sigma}$. We define the output collision probability as
\begin{align}
P \left( (A + \mathbf{\epsilon}) \cap B \neq \varnothing \right) \label{eq:pcd_definition} \\
\mathbf{\epsilon} \sim \mathcal{N}(\mathbf{p}, \mathbf{\Sigma}) . \nonumber
\end{align}
The collision event is equivalent to
\begin{align*}
& (A + \mathbf{\epsilon}) \cap B \neq \varnothing \\
\iff & \exists \mathbf{a} \in A, \mathbf{b} \in B \quad s.t. \quad \mathbf{a} + \mathbf{\epsilon} = \mathbf{b} \\
\iff & \exists \mathbf{a} \in A, \mathbf{b} \in B \quad s.t. \quad \mathbf{\epsilon} = - \mathbf{a} + \mathbf{b} \\
\iff & \exists \mathbf{x} \in (-A) \bigoplus B \quad s.t. \quad \mathbf{\epsilon} = \mathbf{x} \\
\iff & \mathbf{\epsilon} \in (-A) \bigoplus B .
\end{align*}
Then, the probability is formulated as
\begin{align}
& P \left( \mathbf{\epsilon} \in (-A) \bigoplus B \right) \\
=& \iiint I \left( \mathbf{x} \in (-A) \bigoplus B \right) p(\mathbf{x}, \mathbf{0}, \mathbf{\Sigma}) d \mathbf{x} \\
=& \iiint_{V} p(\mathbf{x}, \mathbf{0}, \mathbf{\Sigma}) d \mathbf{x} ,
\label{eq:colobj}
\end{align}
where the function $I(\mathbf x)$ and the obstacle function $p(\mathbf x,\mathbf p,\mathbf \Sigma)$ are defined as,
\begin{align}
\label{eq:colobj1}
I(\mathbf x)=\left\{\begin{matrix}
1 & \textrm{if }\mathbf x \textrm{ is true}\\ 
0 & \textrm{otherwise}
\end{matrix}\right. \, \textrm{and}
\end{align}
\begin{align}
\label{eq:colobj2}
p(\mathbf x,\mathbf p,\mathbf \Sigma)=\frac{e^{-0.5(\mathbf x-\mathbf p)^T\mathbf \Sigma^{-1}(\mathbf x-\mathbf p)}}{\sqrt{(2\pi)^3\|\mathbf \Sigma\|}},
\end{align}
respectively.

\subsection{Collision Probability Approximation}
\label{subsec:cpa}


For a sphere $A$ of radius $r_2$ with $\epsilon \sim \mathcal{N} (\mathbf p, \mathbf \Sigma)$ and a sphere $B$ with radius $r_1$, the exact probability of collision between them is given as the integration of $\epsilon$ in $V$, where $V$ is a sphere or radius $r_1+r_2$.
It is known that there is no closed form solution for the integral given in (\ref{eq:colobj}), even for spheres.

Du Toit and Burdick~\cite{du2011probabilistic} approximate (\ref{eq:colobj}) as 
\begin{align}
P \left( \mathbf{\epsilon} \in ((-A) \bigoplus B) \right) \approx& \iiint_{V} 1 d \mathbf{x} \cdot p(\mathbf{x}, \mathbf{0}, \mathbf{\Sigma})\\
=& \frac{4\pi}{3}(r_1+r_2)^3 \cdot p(\mathbf{x}, \mathbf{0}, \mathbf{\Sigma}).
\end{align}
However, this approximated probability can be either smaller or larger than the exact probability, and it is hard to give any guarantees in terms of lower or upper bound on the computed probability. 

Park et al.~\cite{park2016fast} compute $\mathbf x_{max}$, the position that has the maximum probability of  $\mathcal{N} (\mathbf p, \mathbf \Sigma)$ in $V$, and compute the upper bound of (\ref{eq:colobj}) as 
\begin{align}
\label{eq:approx}
P \left( \mathbf{\epsilon} \in ((-A) \bigoplus B) \right) \approx \frac{4\pi}{3}(r_1+r_2)^3 \cdot p(\mathbf{x_{max}}, \mathbf{0}, \mathbf{\Sigma}),
\end{align}
where $\mathbf x_{max}$ is computed as a solution of 
\begin{align}
\mathbf x_{max} = \argmin_{\mathbf x} \left\{ (\mathbf x-\mathbf p)^T \mathbf \Sigma^{-1}(\mathbf x - \mathbf p)+\lambda \mathbf x^2\right\}
\end{align}
with a Lagrange multiplier $\lambda$, using a one-dimensional numerical search. This approach guarantees that the approximation never underestimates the collision probability, but is limited to isotropic objects such as circles or spheres.

\section{Probabilistic Collision Detection}
\label{sec:pcd}

In this section, we present a fast algorithm for probabilistic collision detection between convex polytopes. Furthermore, we extend to non-convex models using convex decomposition and bounding volume hierarchies.

\subsection{Probabilistic Collision Detection for Convex Polyhedrons}
\label{subsec:convex_pcd}

We extend the probabilistic collision detection for any two 3D convex polyhedrons $A$ and $B$.
We transform the volume $V$ in (\ref{eq:colobj}) by $(\mathbf{\Sigma})^{-\frac{1}{2}}$ to normalize the Gaussian distribution, i.e.,
\begin{align}
& \iiint_{V'} p(\mathbf{x}, \mathbf{0}, \mathbf{I}) d \mathbf{x} \nonumber \\
=& \iiint_{V'} \frac{1}{\sqrt{8 \pi^3}} \exp \left( -\frac{1}{2} \| \mathbf{x} \|^2 \right) d \mathbf{x}, \label{eq:exact_prob}
\end{align}
where $V' = (\mathbf{\Sigma})^{-\frac{1}{2}} V$.

\begin{figure}[ht]
  \centering
  \subfloat[][]
  {
    \includegraphics[width=0.48\linewidth]{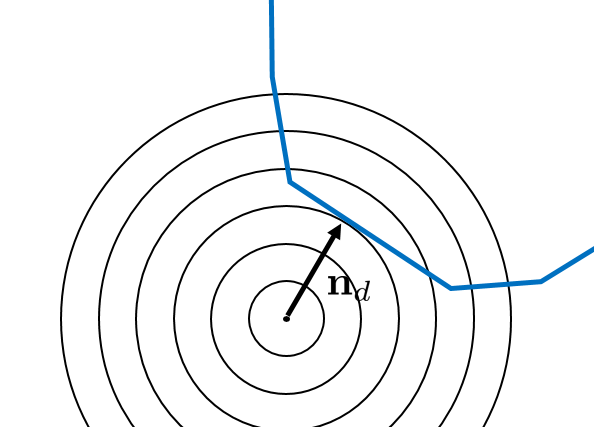}
  }
  \subfloat[][]
  {
    \includegraphics[width=0.48\linewidth]{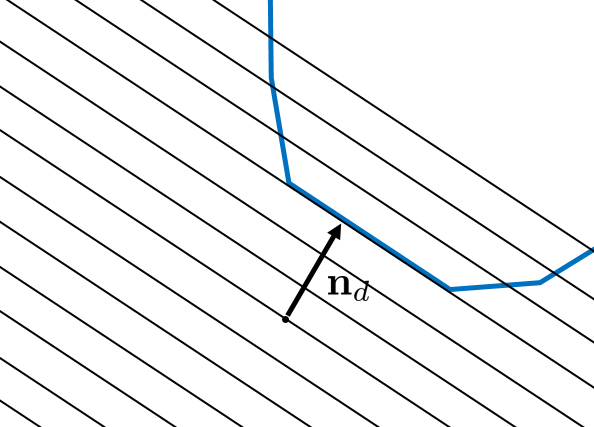}
  }
  \caption{(a) Contour plots of the bivariate Gaussian distribution, whose covariance matrix is normalized to an identity matrix. The minimum displacement vector $\mathbf{n}_d$ from the center to the blue polygon is computed using the GJK algorithm. (b) Contour plots of upperly bounded function $\mathbf{F}$ whose contours are perpendicular to $\mathbf{n}_d$ This function behaves like 1D Gaussian distribution.}
  \label{fig:upper_bound_pdf}
\end{figure}

Evaluating the integration of Gaussian distribution over volume $V'$ is still difficult, but a tight upper bound of the integral can be efficiently computed by replacing the probability distribution function with another one, which makes the integration computation easier. To define another function for computing the upper bound, first we compute the minimum displacement vector $\mathbf{d}$ between convex shapes $A'$ and $B'$ which are transformed from $A$ and $B$, respectively, by $(\mathbf{\Sigma})^{-\frac{1}{2}}$. $\mathbf{d}$ is computed efficiently by the GJK algorithm~\cite{gilbert1988fast}. Let $\mathbf{n}_d$ be the unit directional vector. Then, by the Cauchy-Schwarz inequality,
\begin{align*}
(\mathbf{x} \cdot \mathbf{n}_d)^2 \leq \| \mathbf{x} \|^2.
\end{align*}
Thus,
\begin{align}
& \iiint_{V'} \frac{1}{\sqrt{8 \pi^3}} \exp \left( -\frac{1}{2} \| \mathbf{x} \|^2 \right) d \mathbf{x} \nonumber \\
\leq & \iiint_{V'} \frac{1}{\sqrt{8 \pi^3}} \exp \left( -\frac{1}{2} (\mathbf{x} \cdot \mathbf{n}_d)^2 \right) d \mathbf{x} . \label{eq:pdf_upper}
\end{align}
The benefit of replacing $\| \mathbf{x} \|^2$ with $(\mathbf{x} \cdot \mathbf{n}_d)^2$ is that the latter term, the projection of $\mathbf{x}$ to the direction of $\mathbf{n}_d$, behaves like a 1D parameter instead of 3D, as shown in Fig.~\ref{fig:upper_bound_pdf}.

The divergence theorem is used to compute the upper bound on collision probability (\ref{eq:pdf_upper}).
\begin{align*}
\iiint_{V'} \text{div} (\mathbf{F}) dV = \oiint_{S'} (\mathbf{F} \cdot \mathbf{n}_S) dS ,
\end{align*}
where $\mathbf{F}$ is a vector field, $S'$ is the surface of $V'$, $dS$ is an infinitesimal area for integration, and $\mathbf{n}_S$ is the normal vector of $dS$. We want the divergence of $\mathbf{F}$ to be equal to the function inside the integral in (\ref{eq:pdf_upper}). Let's define $\mathbf{F}$ as
\begin{align*}
\mathbf{F}(\mathbf{x}) = \frac{1}{2 \pi} \left( 1 + \text{erf} \left( \frac{\mathbf{x} \cdot \mathbf{n}_d}{\sqrt{2}} \right) \right) \mathbf{n}_d ,
\end{align*}
where $\text{erf()}$ is the Gaussian error function. The directional derivative of $\mathbf{F}(\mathbf{x})$ along any directional vector orthogonal to $\mathbf{n}_d$ is zero because $\mathbf{F}$ varies only along $\mathbf{n}_d$. The divergence of $\mathbf{F}$ thus becomes $(\partial \mathbf{F} / \partial \mathbf{n}_d)$, and this is equal to the function in (\ref{eq:pdf_upper}).

The right-hand side of the divergence theorem makes the integration efficient for a convex polyhedron $S'$. It is formulated as
\begin{align*}
\sum_{i} \oiint_{\triangle S_{i1} S_{i2} S_{i3}} (\mathbf{F} \cdot \mathbf{n}_i) dS ,
\end{align*}
where $\mathbf{n}_i$ is the normal vector of the $i$-th triangle $\triangle S_{i1} S_{i2} S_{i3}$. Because the error function integral over a triangle domain is another hard problem, the upper bound on the integral is evaluated instead, that is
\begin{align}
& \oiint_{\triangle S_{i1} S_{i2} S_{i3}} (\mathbf{F} \cdot \mathbf{n}_i) dS \nonumber \\
\leq & \sum_i \left( \max_{j=1,2,3} \mathbf{F}(S_{ij}) \cdot \mathbf{n}_i \right) \text{Area}(\triangle S_{i1} S_{i2} S_{i3}) . \label{eq:upper_bound}
\end{align}
Summing up the values for all triangles gives a tight upper bound of collision probability, as defined in (\ref{eq:pcd_definition}).

This gives higher value than expected collision probability, as compared to that computed using Monte Carlo methods. Monte Carlo methods take $n_{MC}$ samples, $\epsilon_i$, from position error $\mathcal{N}(\mathbf{0}, \Sigma)$.
\begin{theorem}
\label{thm:monte_carlo}
{\em As the number of Monte Carlo increases to the infinity, the approximated probability by Monte Carlo method is upperly bounded by Equation (\ref{eq:upper_bound}).}
\end{theorem}
\begin{proof}
The collision probability approximated by Monte Carlo methods follows a binomial distribution $B(n_{MC}, p^*)$ divided by $n_{MC}$:
\begin{align*}
\label{eq:monte_carlo}
\frac{1}{n_{MC}} I \left( \mathbf{\epsilon_i} \in ((-A) \bigoplus B) \right) \sim \frac{1}{n_{MC}} B(n_{MC}, p^* ) ,
\end{align*}
where $B(n, p)$ is a binomial distribution and $p^*$ is the exact collision probability given as (\ref{eq:exact_prob}).
The expectation of Monte Carlo approximation is
\begin{align*}
 & \mathrm{E} \left[ \frac{1}{n_{MC}} B(n_{MC}, p^* ) \right] \\
=& \frac{1}{n_{MC}} \mathrm{E} \left[ B(n_{MC}, p^* ) \right] \\
=& \frac{1}{n_{MC}} n_{MC} p^* = p^* \\
=& \iiint_{V'} \frac{1}{\sqrt{8 \pi^3}} \exp \left( -\frac{1}{2} \| \mathbf{x} \|^2 \right) d \mathbf{x} \\
\leq & \iiint_{V'} \frac{1}{\sqrt{8 \pi^3}} \exp \left( -\frac{1}{2} (\mathbf{x} \cdot \mathbf{n}_d)^2 \right) d \mathbf{x} \\
=& \sum_{i} \oiint_{\triangle S_{i1} S_{i2} S_{i3}} (\mathbf{F} \cdot \mathbf{n}_i) dS \\
\leq & \sum_i \left( \max_{j=1,2,3} \mathbf{F}(S_{ij}) \cdot \mathbf{n}_i \right) \text{Area}(\triangle S_{i1} S_{i2} S_{i3}) .
\end{align*}
The variance of Monte Carlo approximation converges to zero:
\begin{align*}
 & \mathrm{Var} \left[ \frac{1}{n_{MC}} B(n_{MC}, p^* ) \right] \\
=& \frac{1}{n_{MC}^2} \mathrm{Var} \left[ B(n_{MC}, p^* ) \right] \\
=& \frac{1}{n_{MC}^2} n_{MC} p^* (1 - p^*) \\
=& \frac{1}{n_{MC}} p^* (1 - p^*) \xrightarrow{n_{MC} \rightarrow \infty} 0 .
\end{align*}
Thus, the Monte Carlo approximation converges to the exact collision probability as the number of samples increases, and the approximation value is bounded by Equation (\ref{eq:upper_bound}).
\end{proof}

\begin{algorithm}[t]
  \caption{$p_{col}$ = ConvexPCD($A$, $B$, $\mathbf{\Sigma}$) \\: Compute the upper bound on collision probability between convex polytopes $A$ and $B$, given a covariance matrix $\mathbf{\Sigma}$.}
  \label{alg:convex}
  \begin{algorithmic}[1]
    \REQUIRE two convex shapes $A$ and $B$,
    \ENSURE Upper bound on collision probability $p_{col}$ \label{line:convex:fast_start}
    \STATE $A'$ = $(\mathbf{\Sigma})^{-\frac{1}{2}} A$
    \STATE $B'$ = $(\mathbf{\Sigma})^{-\frac{1}{2}} B$
    \STATE $\mathbf{n}_d$ = GJK($A'$, $B'$) to define $\mathbf{F}$ \label{line:convex:fast_end}
    \STATE $V'$ = MinkowskiSum($A'$, $B'$) \label{line:convex:minkowski_sum}
    \STATE $p_{col}$ = $0$
    \FORALL {$i$ of $V'$} \label{line:convex:loop_start}
      \STATE Add Equation (\ref{eq:upper_bound}) to $p_{col}$
    \ENDFOR \label{line:convex:loop_end}
    \RETURN $p_{col}$
  \end{algorithmic}
\end{algorithm}

Algorithm~\ref{alg:convex} describes how the upper bound on collision probability for two convex shapes is computed. The two input shapes are transformed so that the variance of Gaussian distribution becomes isotropic. Then, GJK algorithm is used for finding the minimum displacement vector $\mathbf{\Sigma}$ between $A'$ and $B'$. The Minkowski sum $V'$ is computed and then the upper bound on collision probability is directly computed from its all triangles. Computing the Minkowski sum (line~\ref{line:convex:minkowski_sum}) has a time complexity linearly proportional the number of output vertices. The loop (line~\ref{line:convex:loop_start}-\ref{line:convex:loop_end}) also has a time complexity linearly proportional to the number of output vertices. In the worst case, the time complexity can be $O(n^2)$, where the number of vertex in $A'$ and $B'$ is $O(n)$. However, special types of convex shapes can be used for efficient computation. k-DOPs have $k$ planes surrounding the objects, and the Minkowski sum computation takes $O(k)$ instead of $O(k^2)$, because the triangles are parallel to corresponding ones. AABB is special case of k-DOPs with $k=3$. OBB has 3 pairs of parallel rectangular faces, so their Minkowski sum can be efficiently computed. As a result, we have much simpler and faster algorithms to compute the collision probability for these widely used bounding volumes.

\subsection{Probabilistic Collision Detection for General Shapes}

In order to handle non-convex objects efficiently, we decompose them into many convex shapes and build bounding volume hierarchies (BVHs) to reduce the number of operations for every convex-convex shape pair.

\begin{figure}[ht]
  \centering
  \subfloat[][]
  {
    \includegraphics[width=0.7\linewidth]{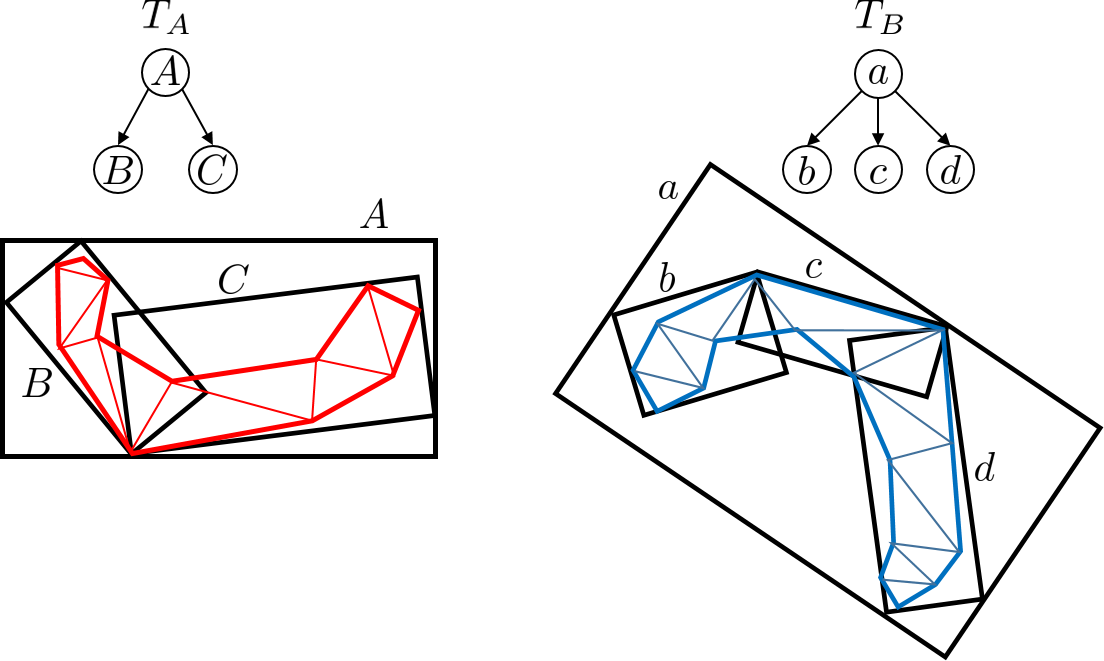}
  }
  \\
  \subfloat[][]
  {
    \includegraphics[width=0.40\linewidth]{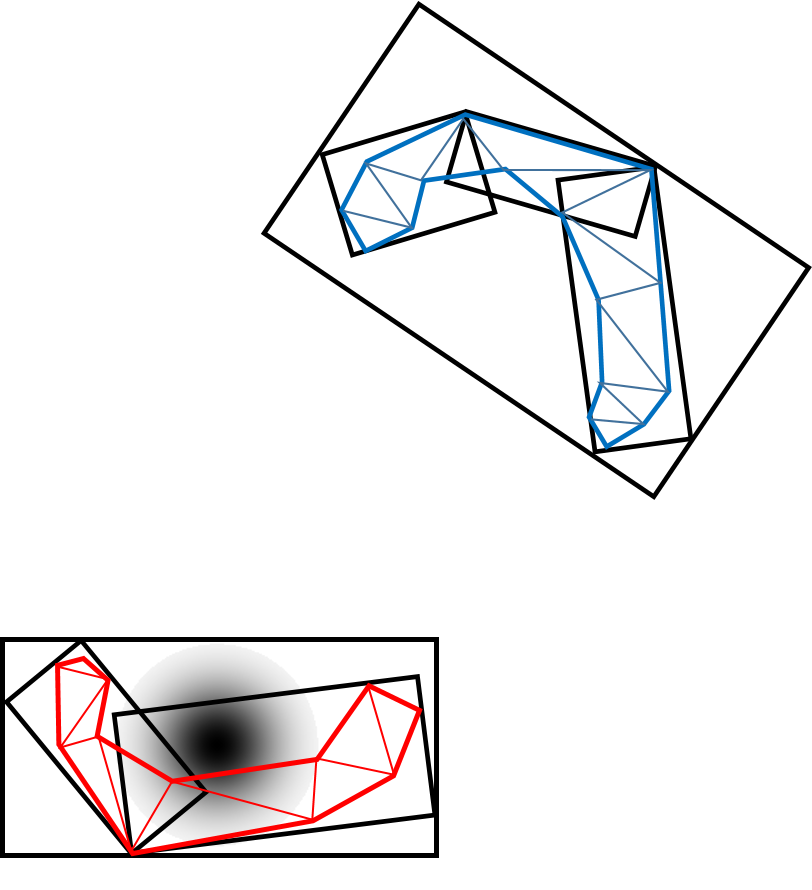}
  }
  \subfloat[][]
  {
    \includegraphics[width=0.40\linewidth]{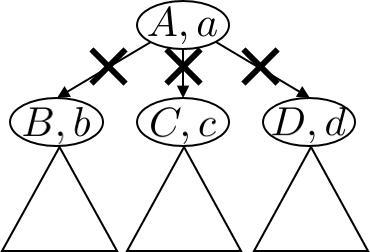}
  }
  \\
  \subfloat[][]
  {
    \includegraphics[width=0.40\linewidth]{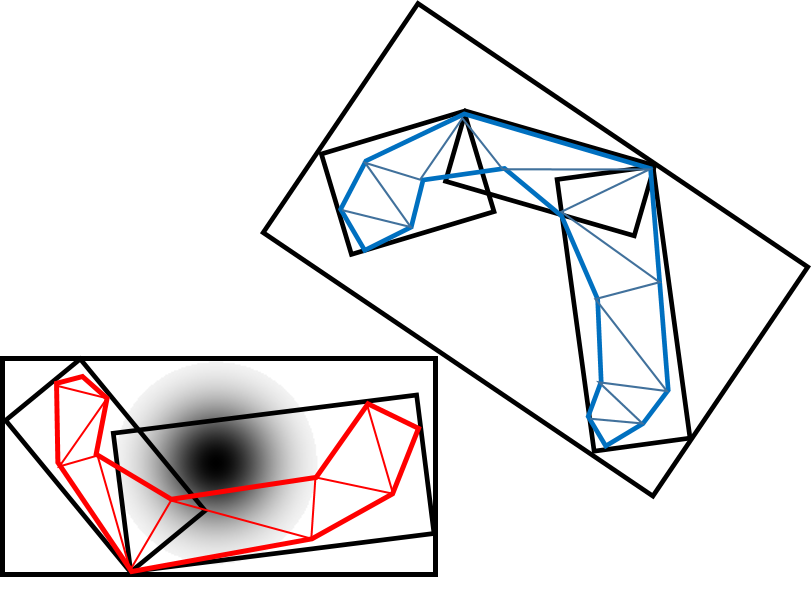}
  }
  \subfloat[][]
  {
    \includegraphics[width=0.40\linewidth]{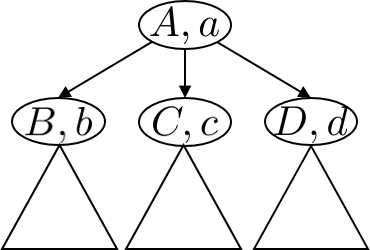}
  }
  \caption{A 2D example of BVHs and their traversal during probabilistic collision detection. The non-convex shapes are shown as red and blue polygons, and they are bounded by OBB trees in (a). The probability distribution of displacement of the red object is drawn at the center of the bounding box in (b) and (d). BVH traversal (b-c) stops or (d-e) continues, which is determined by comparing the upper bound on collision probability and the confidence level $\delta_{CD}$.}
  \label{fig:bvh_traversal}
\end{figure}

We define a confidence level $\delta_{CD} < 1$ (usually higher than $0.90$), an indication of when the traversal in BVHs stops. The confidence level means that we can confidently say that there will likely be no collision if the collision probability is less than $1 - \delta_{CD}$. To be more specific, the upper bound on the collision probability of two bounding volumes of current traversing nodes is computed as described in Section~\ref{subsec:convex_pcd}, and is compared to $1 - \delta_{CD}$. If the collision probability is less than $1 - \delta_{CD}$, it means that the probability of collision at this level of bounding volume is sufficiently low. Otherwise, traversing further to the child nodes is needed to compute a tighter bound on collision probability.

Fig.~\ref{fig:bvh_traversal} shows two cases in BVH traversal: stop and continue. In Fig.~\ref{fig:bvh_traversal}(b), the two bounding volumes are so far away that the upper bound on collision probability is already sufficiently below the threshold $1 - \delta_{CD}$. Therefore, traversal stops in the BVHs in Fig.~\ref{fig:bvh_traversal}(c), resulting in the reduction in traversing. On the other hand, in Fig.~\ref{fig:bvh_traversal}(d), the two bounding volumes are so close that they are likely to collide with each other, though the actual shapes are not so close. In this case, the traversal continues to the children for tighter bound computation.

Algorithm~\ref{alg:general} describes how the BVHs are traversed and when the traversal stops to reduce the number of computations. Two BVHs $T_A$ and $T_B$, corresponding to any type of bounding volumes (e.g., spheres, AABBs, OBBs, k-DOPs, or convex hulls), are traversed simultaneously from roots to leaf nodes. The upper bound on collision probability is computed for the convex volume of current nodes at line~\ref{line:general:convex_pair} as described in Section~\ref{subsec:convex_pcd}. The traversal stops when the collision probability is sufficiently low compared to $1 - \delta_{CD}$ or it has reached the leaf nodes (line~\ref{line:general:treversal_stop}). Otherwise, traversal continues to the child nodes of $T_A$ or $T_B$, depending on the size of two bounding volumes, in lines~\ref{line:general:traversal_start}-\ref{line:general:traversal_end}.

\begin{algorithm}[t]
  \caption{$p_{col} = $GeneralPCD($T_A$, $T_B$, $\mathbf{\Sigma}$, $\delta_{CD}$) \\: Compute the upper bound on collision probability between non-convex polytopes $A$ and $B$, given BVHs $T_A$ and $T_B$, a covariance matrix $\mathbf{\Sigma}$ and a confidence level $\delta_{CD}$.}
  \label{alg:general}
  \begin{algorithmic}[1]
    \REQUIRE two BVHs $T_A$ and $T_B$, confidence level $\delta_{CD}$,
    \ENSURE Upper bound on collision probability $p_{col}$
    \STATE $v_A$ = $T_A$.root
    \STATE $v_B$ = $T_B$.root
    \STATE $p_{root}$ = ConvexPCD($v_A$, $v_B$, $\mathbf{\Sigma}$)  \label{line:general:convex_pair}
    \IF {$p_{root} < 1 - \delta_{CD}$ or both $v_A$ and $v_B$ are leaf nodes} \label{line:general:treversal_stop}
      \RETURN $p_{col} = p_{root}$
    \ENDIF
    \IF {$v_A$ has children and $Vol(v_A) \geq Vol(v_B)$} \label{line:general:traversal_start}
      \RETURN $p_{col} = \sum\limits_{c_A : \text{child}}$GeneralPCD($c_A$, $T_B$, $\mathbf{\Sigma}$, $\delta_{CD}$)
    \ELSE
      \RETURN $p_{col} = \sum\limits_{c_B : \text{child}}$GeneralPCD($T_A$, $c_B$, $\mathbf{\Sigma}$, $\delta_{CD}$)
    \ENDIF \label{line:general:traversal_end}
  \end{algorithmic}
\end{algorithm}

\section{Results}
\label{sec:result}
\vspace*{-0.1in}

\begin{figure}[ht]
  \centering
  \subfloat[][Benchmark \#1\\(5,110 triangles)]
  {
    \includegraphics[width=0.32\linewidth]{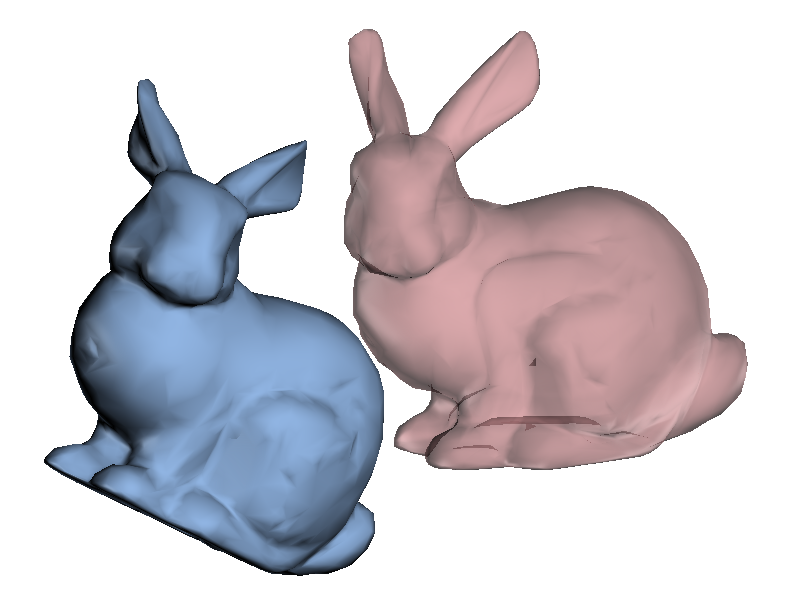}
  }
  \subfloat[][Benchmark \#2\\(10,000 triangles)]
  {
    \includegraphics[width=0.32\linewidth]{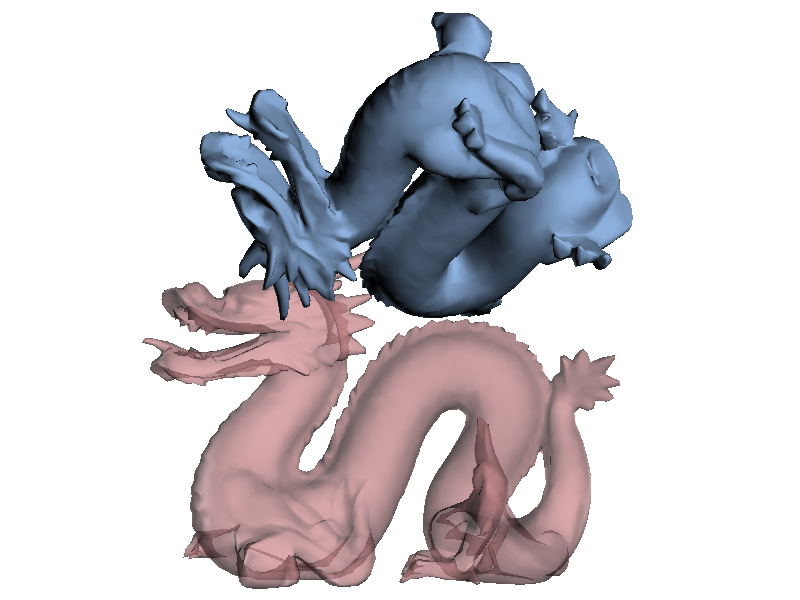}
  }
  \subfloat[][Benchmark \#3\\(10,000 triangles)]
  {
    \includegraphics[width=0.32\linewidth]{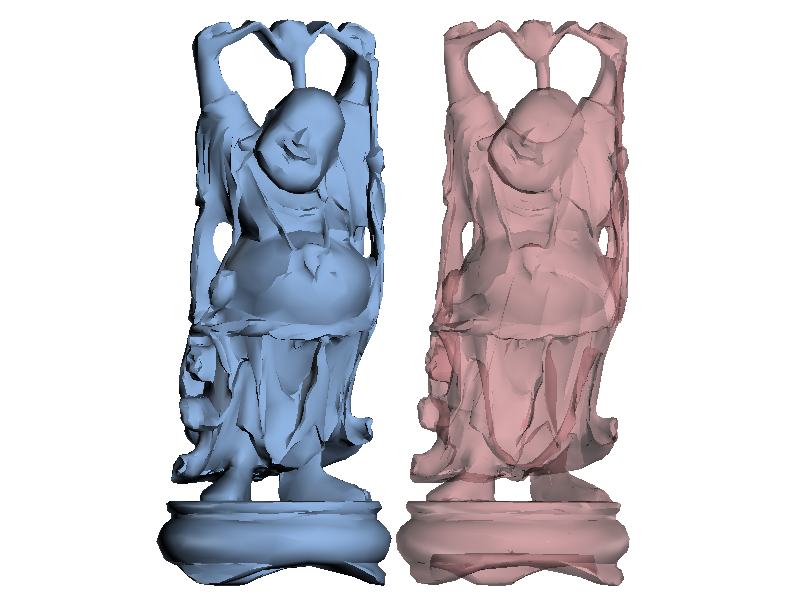}
  }
  \caption{{\bf Geometry Models and Obstacle Position Distributions.} We evaluate the performance of different algorithms with different non-convex shapes: (a) Benchmark \#1: bunny, (b) Benchmark \#2: dragon, and (c) Benchmark \#3: buddha. The red objects have Gaussian distribution errors in position.}
  \label{fig:geometries}
\end{figure}

In this section, we describe our implementation and  highlight the performance of our probabilistic collision detection algorithms on synthetic and real-world benchmarks. Furthermore, we compare the performance of different bounding volumes (i.e. spheres, AABBs, OBBs, k-DOPs, convex hull) in terms of the probability bounds and the query time. We use the following performane measurements to evaluate the performance:

\begin{itemize}
\item \textbf{Bounding volume approximation error (BVE):}. It is measured as the ratio of the volume of bounding shapes to that of the original or underlying shapes, which is always greater than or equal to 1. The closer this ratio is to 1, the better the bounding volume approximates the original shape. This ratio is low for tight fitting bounding volumes such as convex hull and high for spheres.
\item \textbf{Collision probability upper bound (CP):}. We compute the upper bound on collision probability between two nearby shapes based on Gaussian probability distribution. The actual collision probability cannot be computed analytically, so it is approximated by the Monte Carlo method. In practice, Monte Carlo methods can take a long time, but we assume that they provide the most accurate solutions. The upper bounds computed by our algorithms are expected to be higher than the Monte Carlo approximation. Our goal is to efficiently compute the tightest upper bound on collision probability with different bounding volumes for non-convex shapes.
\item \textbf{Computation time (Time):}. Each algorithm has a different time complexity, so it is important to evaluate the trade-off between the tightness of a collision probability upper bound and time complexity. All the timings reported in this paper are measured in milliseconds (ms) and generated on a single CPU core of Intel I7 processor.
\end{itemize}

The Monte Carlo method provides an approximated value of the actual collision probability by sampling $10,000$ points from the obstacle's position distribution in our benchmarks. We expect that is the most accurate available answer for collision probability computations. As a result, all the  collision probability computations computed using our approaches compute a higher value than Monte Carlo approximation, but takes significantly less time than Monte Carlo integration. For AABBs, different global coordinate systems can result in different AABB shapes. As a result, we randomly generate $10,000$ global coordinate systems from the $SO(3)$ group to generate the axes and compute the averages over these coordinates.

\subsection{Complex Non-convex Models}

We measure these values for different non-convex shapes shown in Fig.~\ref{fig:geometries}. These are very complex synthetic benchmarks in close proximity that are used to evaluate different collision detection algorithms. These models are scales so that the length of the longest bounding AABB is  $1$m. The positions are determined so that the minimum distance between two shapes is $1$ cm or $5$ cm, and we compute the collision probabilities for such configurations. The covariance matrix of position error, $\mathbf{\Sigma}$, is set with a random orientation as its main axes and use 1cm, 3cm, and 5cm as the standard deviation along the axes.

\begin{table*}[h]
\centering
\scalebox{1.0}{
\begin{tabular}{|c|c|c|c|c|c|c|c|c|c|c|c|c|c|c|c|}
\hline
\multirow{3}{*}{BVH type} & \multicolumn{3}{|c|}{\multirow{2}{*}{BVE}} & \multicolumn{6}{|c|}{distance between shapes : 1 cm} & \multicolumn{6}{|c|}{distance between shapes : 5 cm} \\ \cline{5-16}
& \multicolumn{3}{|l|}{} & \multicolumn{3}{|c|}{CP (\%)} & \multicolumn{3}{|c|}{Computation Time (ms)} & \multicolumn{3}{|c|}{CP (\%)} & \multicolumn{3}{|c|}{Computation Time (ms)} \\ \cline{2-16} 
& \#1 & \#2 & \#3 & \#1 & \#2 & \#3 & \#1 & \#2 & \#3 & \#1 & \#2 & \#3 & \#1 & \#2 & \#3 \\ \hline
\begin{tabular}[x]{@{}c@{}}Monte Carlo\\integration\end{tabular}
                        & 1.00 & 1.00 & 1.00 & 13.2 & 15.7 & 6.82 & 54,300 & 213,000 & 108,000 & 2.38 & 3.72 & 0.89 & 32,800 & 176,000 & 92,000 \\ \hline
Spheres                 & 4.80 & 6.28 & 6.50 & 45.7 & 52.8 & 34.9 &    847 & 2,480   & 2,070   & 10.3 & 18.1 & 10.3 & 261    & 1,580 & 1,030 \\ \hline
AABBs                    & 3.28 & 5.20 & 4.96 & 32.6 & 38.1 & 31.7 &    190 & 578     & 420     & 6.21 & 11.9 & 7.06 & 67     & 213   & 126   \\ \hline
OBBs                     & 1.63 & 3.16 & 2.94 & 20.3 & 31.7 & 28.8 &    179 & 580     & 441     & 4.20 & 6.23 & 2.00 & 86     & 278   & 158   \\ \hline
26-DOPs                 & 1.26 & 1.68 & 1.88 & 16.7 & 22.3 & 12.7 &    897 & 3,320   & 2,637   & 3.57 & 2.90 & 1.27 & 326    & 1,010 & 868   \\ \hline
Convex                  & 1.00 & 1.00 & 1.00 & 14.2 & 18.9 & 8.09 &  3,721 & 10,800  & 8,780   & 2.70 & 4.18 & 1.12 & 1,082  & 2,520 & 2,330 \\ \hline
\end{tabular}
}
\caption{Performance of different algorithms for complex Benchmarks \#1-3 shown in Fig. 4. We measure the bounding volume approximation error (BVE), collision probability upper bound (CP), and query computation time for different bounding volumes. The Monte Carlo integration schemes provides the most accurate result, but is very expensive. On the other hand, spheres (used in prior methods~\cite{du2011probabilistic,park2016fast}) do not provide tight bounds, as observed with CP values. In practice, OBBs seem to provide the best balance between collision probability bounds and query times.}
\label{table:performance_models}
\end{table*}

\begin{figure}[ht]
  \centering
  \subfloat[][1 cm]
  {
    \includegraphics[trim=0.8in 1in 1in 0.8in, clip=true, width=0.48\linewidth]{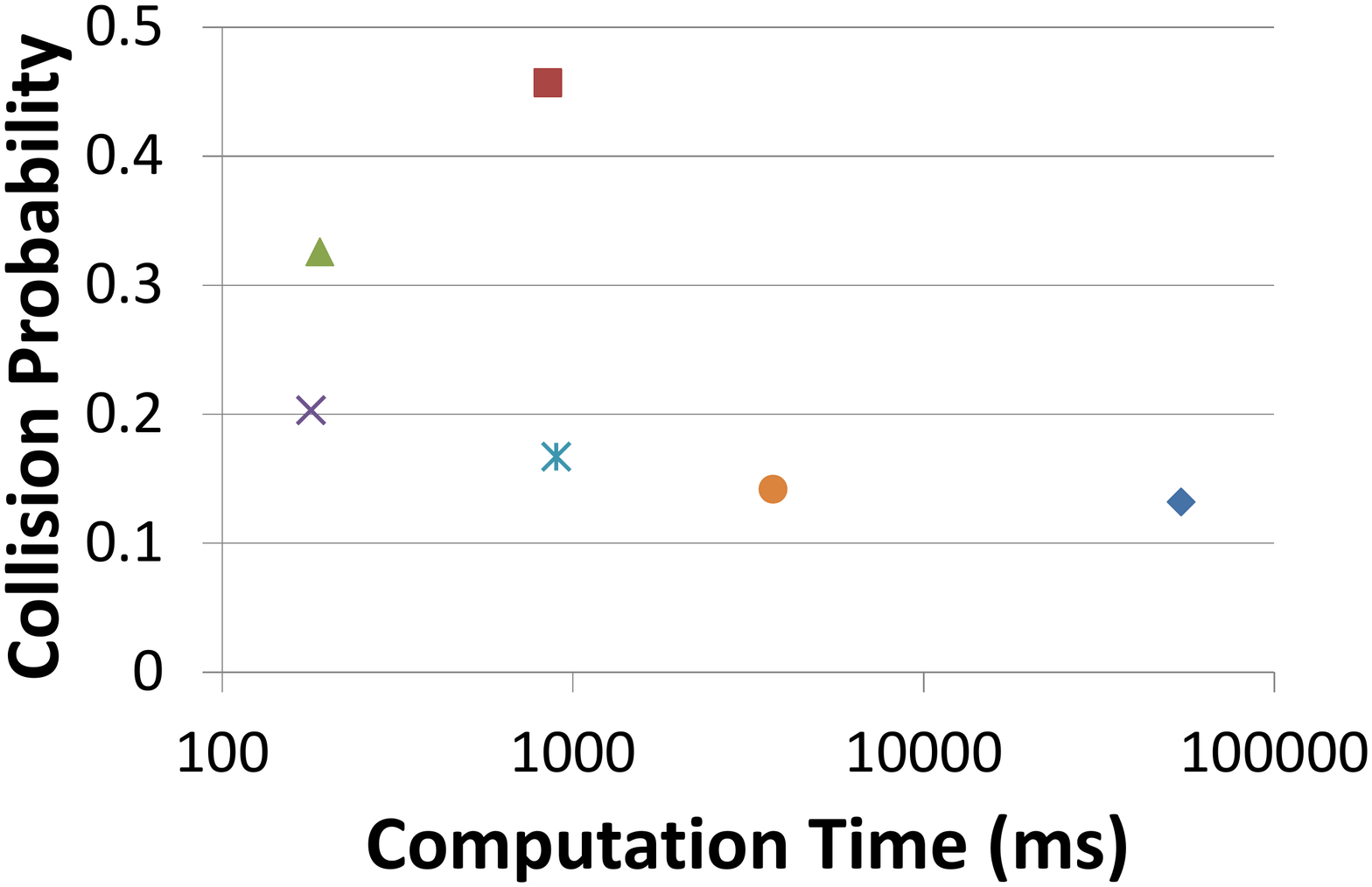}
  }
  \subfloat[][5 cm]
  {
    \includegraphics[trim=0.8in 1in 1in 0.8in, clip=true, width=0.48\linewidth]{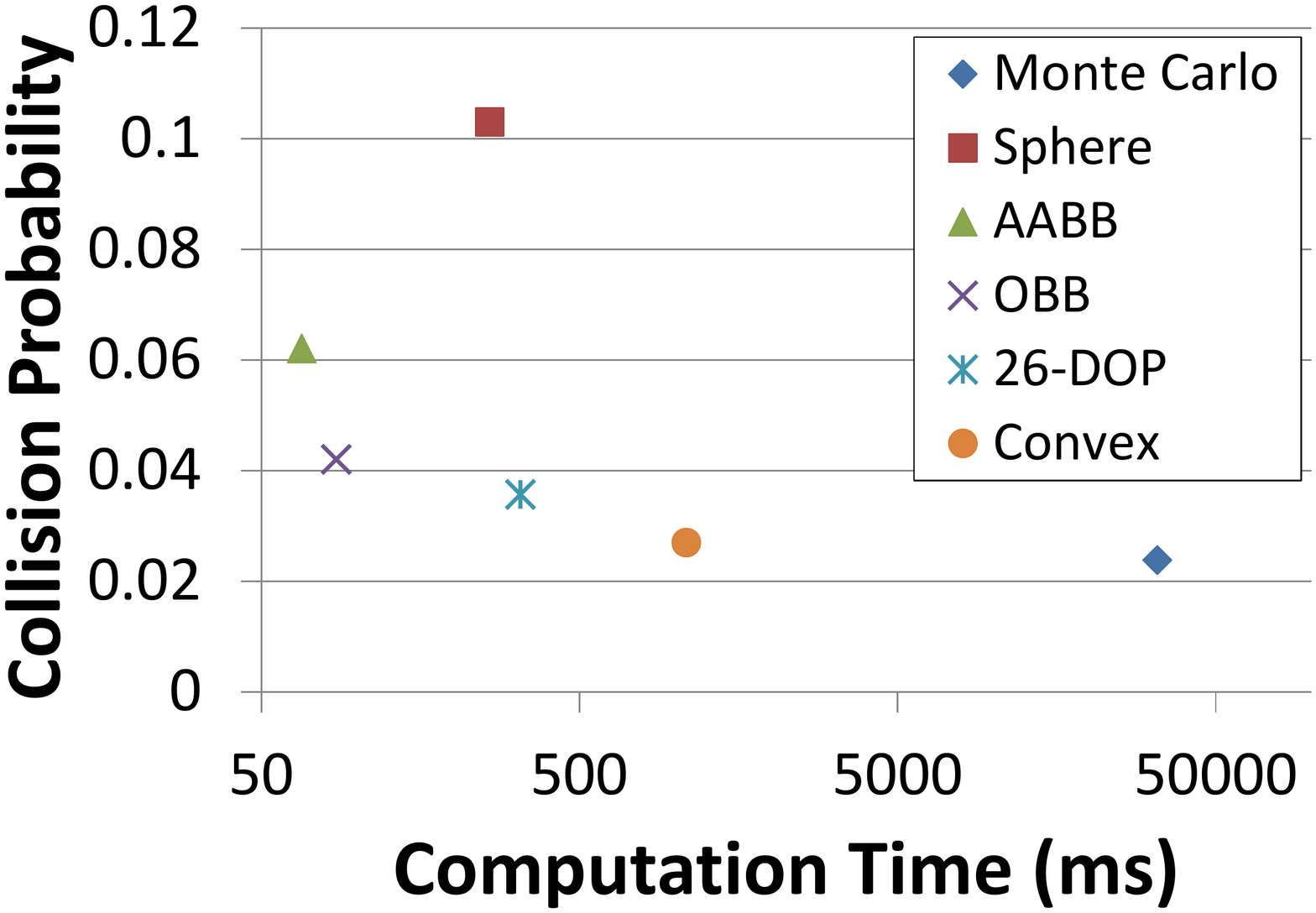}
  }
  \caption{Performance graph showing the computation time (x-axis) vs. the upper bound on collision probability (y-axis) between the two bunny models away from (a) 1 cm and (b) 5 cm.}
  \label{fig:graph_model}
\end{figure}

Table~\ref{table:performance_models} and Fig.~\ref{fig:graph_model} show the upper bound on collision probability and the query time using different bounding volumes with different approximation errors. We notice that the OBBs seem to provide the best balance between probability bound and the running time. In case of spheres~\cite{du2011probabilistic,park2016fast}), the running time is higher than AABBs or OBBs, because the culling efficiency of sphere is low because the collision probability bounds are not tight. As a result, the algorithm traverses more nodes in the hierarchy. As compared to OBBs, 26-DOPs provide tighter bounds on the collision probability. However, the query time on 26-DOPs is higher than AABBs and OBBs, because the overhead of computing the Minkowski sum.The performance of OBB is comparable to AABB, but OBBs provide a tighter bound on the probability computation.

\subsection {Robot Trajectory Planning with Sensor Errors}
We have integrated our probabilistic collision algorithms with trajectory planning algorithm and evaluated their performance on a 7-DOF Fetch arm.
In our experiments, we compute different bounding volumes for the robot and the obstacles in the scene. Furthermore, the scene consists of a dynamic human obstacles (see Fig. 1) and we assume that the robot is operating in a close proximity to the human. In this case, the robot predicts the trajectory of the human, represents that with a Gaussian distribution and uses probabilistic collision checking to compute a collision-free trajectory.  
The robot uses a  Kinect as the depth sensor, which can represent the human using with $512 \times 424$ points. We also compute the state of human obstacle model, which is represented using $60$ DOFs.
It is assumed that the obstacle shape is known in our implementation. The goal of robot trajectory planning is to generate a path where the collision probability is less than a user-specified safety level. In our benchmarks, the safety level is set to $5\%$, meaning that the collision probability between an obstacle and a robot state at any time along the trajectory should be less than $5\%$. A tighter bound computed using probabilistic collision detection algorithm increases the search space of the trajectory planning algorithm.

\begin{figure}[ht]
  \centering
  \subfloat[][]
  {
    \includegraphics[width=0.32\linewidth]{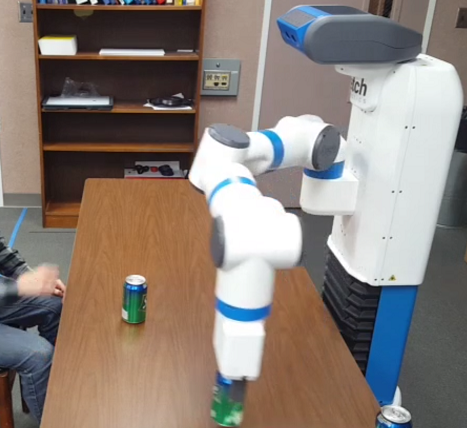}
  }
  \subfloat[][]
  {
    \includegraphics[width=0.32\linewidth]{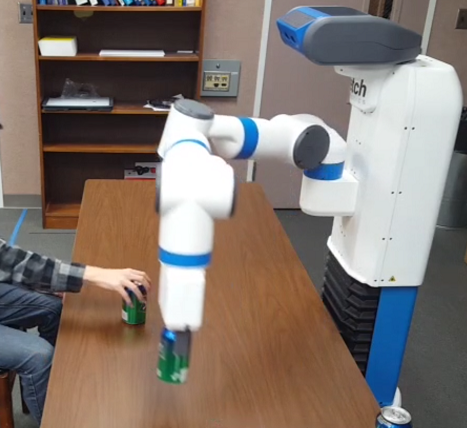}
  }
  \subfloat[][]
  {
    \includegraphics[width=0.32\linewidth]{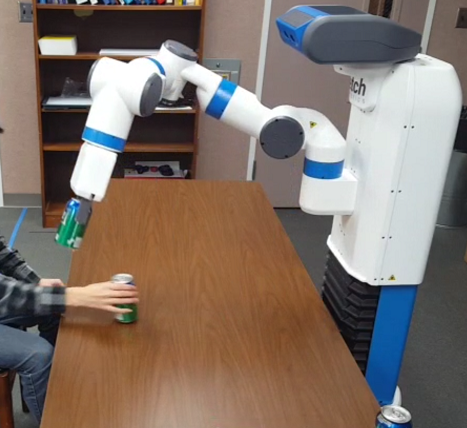}
  }
  \caption{{\bf Trajectory Planning with a 7-DOF robot:} The robot is used to move the coke cans on a table while avoiding a human user, i.e. the human arm is close to the robot. The robot trajectory is computed using a planner that uses our probabilistic collision detection.  At any time along the robot trajectory, the collision probability is less than $5\%$.}
  \label{fig:bencmark_coke_cans}
\end{figure}

\begin{table*}[h]
\centering
\scalebox{1.0}{
\begin{tabular}{|c|c|c|c|c|}
\hline
\multirow{2}{*}{BVH type} & \multicolumn{2}{|c|}{Moving coke cans} & \multicolumn{2}{|c|}{Waving an arm} \\ \cline{2-5}
& CP (\%) & Time (ms) & CP (\%) & Time (ms) \\ \hline
Monte Carlo & 1.00 & 20,800 & 1.00 & 15,200 \\ \hline
Sphere      & 3.81 &  3,260 & 4.72 &  1,730 \\ \hline
AABB        & 3.72 &    637 & 3.26 &    425 \\ \hline
OBB         & 2.50 &    819 & 2.80 &    433 \\ \hline
26-DOPs     & 2.11 &  1,280 & 2.71 &    823 \\ \hline 
Convex      & 1.59 &  8,440 & 1.39 &  4,630 \\ \hline
\end{tabular}
}
\caption{Robot motion planning scenarios using collision probability computation as a constraint in terms of trajectory planning.   The upper bound on collision probability between an obstacle and the robot trajectory, computed using different bounding volume types, should be less than $5\%$.  The Sphere BVH computes the most conservative bound. We evaluated the exact collision probability and computation time for each trajectory for different bounding volume hierarchies in this real-world scenario. In these scenarios, OBBs provide the best balance between collision bounds and the query time.}
\label{table:performance_robot}
\end{table*}

\section{Conclusions, Limitations and Future Work}

In this paper, we present fast and reliable probabilistic collision detection algorithms for general convex and non-convex shape models. This includes an efficient algorithm for convex polytopes that is based on computing the Minkowski sums pf two polytopes. We show that the probability bound computed by our approach is always an upper bound. We also simplify the computations and present optimized algorithms for simpler convex shapes such as AABBs, OBBs and convex hulls. Based on these bounding volumes, we present a hierarchical algorithm for non-convex shapes. We have evaluated their performance on complex synthetic benchmarks and also integrated them with a real-time trajectory planning algorithm. In practice, we observe that OBBs seem to present the best balance between the tightness of bounds and the query times.

Our approach has some limitation. Current formulation is designed for rigid model and the error is represented using a Gaussian probability distribution. The performance of different bounding volumes can vary based on the shape of the objects and their relative configurations.  Furthermore, we only take into account only the position error in our benchmarks, and not the orientation error.
There are many avenues for future work.  Besides overcoming these limitations, we will like to design efficient algorithms with tight bounds for articulated models. It would also be useful to derive similar algorithms for other distributions and evaluate their performance in real-world scenarios.


\bibliographystyle{IEEEtran}
\bibliography{icra17}

\end{document}